\documentclass{article}

\usepackage[accepted,arxiv]{icml2020}
\usepackage{lmodern}
\usepackage[english]{babel}
\usepackage{latexsym}
\usepackage{amsmath}
\usepackage{mathrsfs}
\usepackage{amssymb}
\usepackage{mathtools}
\usepackage[inline,shortlabels]{enumitem} % we prefer enumitem because of its margin adjustment caps
\usepackage{bm}
\usepackage{datetime}
\usepackage[table,xcdraw]{xcolor}
\usepackage{accents}
\usepackage{tikz}
\usepackage{listings}
\usepackage{mdframed}
\usepackage{pgfplots}
\usepackage{pgfplotstable}
\usepackage{dsfont}
\usepackage{color}
\usepackage{colortbl}
\usepackage{pifont}
\usepackage[bf,font=small,singlelinecheck=off]{caption}
\usepackage[disable,textsize=tiny]{todonotes}
\usepackage{microtype} % improved spacing between words for easier reading
\usepackage{float}
\usepackage{xfrac} % sfrac
\usepackage{xspace}

\makeatletter
\renewcommand{\todo}[2][]{\@todo[#1]{#2}}
\makeatother

\usepackage{hyperref}
\hypersetup{
    bookmarks=true,         % show bookmarks bar?
    unicode=false,          % non-Latin characters in AcrobatÕs bookmarks
    pdftoolbar=true,        % show AcrobatÕs toolbar?
    pdfmenubar=true,        % show AcrobatÕs menu?
    pdffitwindow=false,     % window fit to page when opened
    pdfstartview={FitH},    % fits the width of the page to the window
    pdftitle={Learning with Good Feature Representations},    % title
    pdfauthor={},     % author
    pdfsubject={Bandits, Reinforcement Learning},   % subject of the document
    pdfcreator={Creator},   % creator of the document
    pdfproducer={Producer}, % producer of the document
    pdfkeywords={bandits} {reinforcement learning} {linear function approximation}, % list of keywords
    pdfnewwindow=true,      % links in new window
    colorlinks=true,       % false: boxed links; true: colored links
    linkcolor=blue,          % color of internal links (change box color with linkbordercolor)
    citecolor=blue,        % color of links to bibliography
    filecolor=magenta,      % color of file links
    urlcolor=cyan           % color of external links
}
\usepackage{amsthm}
\usepackage{times}
\usepackage{natbib}
\usepackage{nicefrac}
\usepackage{wrapfig}
\usepackage[capitalize]{cleveref}
\usepackage[nottoc,numbib]{tocbibind}
\usepackage[normalem]{ulem}

\theoremstyle{plain}
\newtheorem{theorem}{Theorem}[section]
\newtheorem{proposition}[theorem]{Proposition}
\newtheorem{lemma}[theorem]{Lemma}
\newtheorem{corollary}[theorem]{Corollary}
\theoremstyle{definition}
\newtheorem{definition}[theorem]{Definition}
\newtheorem{assumption}[theorem]{Assumption}
\newtheorem{remark}[theorem]{Remark}
\theoremstyle{remark}

\newcommand{\E}{\mathbb E}
\newcommand{\EE}[1]{\mathbb E[#1]}
\newcommand{\Supp}{\operatorname{supp}}
\newcommand{\ip}[1]{\langle #1 \rangle}
\newcommand{\bip}[1]{\left\langle #1 \right\rangle}
\newcommand{\norm}[1]{\|#1\|}
\newcommand{\R}{\mathbb{R}}

\newcommand{\cA}{\mathcal{A}}
\newcommand{\cC}{\mathcal{C}}

\newcommand{\cE}{\mathcal{E}}

\newcommand{\cH}{\mathcal{H}}

\newcommand{\cR}{\mathcal{R}}

\newcommand{\sA}{\mathscr A}
\newcommand{\sF}{\mathscr F}

\newcommand{\epsapp}{\epsilon}
\newcommand{\epssub}{\delta}
\newcommand{\est}{\operatorname{est}}

\DeclareMathOperator{\Range}{range}
\newcommand{\rows}{\operatorname{rows}}

\renewcommand{\epsilon}{\varepsilon}
\newcommand{\ceil}[1]{\left\lceil {#1} \right\rceil}
\newcommand{\floor}[1]{\left\lfloor {#1} \right\rfloor}

\newcommand{\zeros}{\mathbf{0}}
\DeclareMathOperator*{\argmin}{arg\ min}
\DeclareMathOperator*{\argmax}{arg\ max}

\setlist{nosep}

\newif\ifsup\suptrue

\allowdisplaybreaks

\begin{document}

\twocolumn[
\icmltitle{Learning with Good Feature Representations in Bandits and in RL with a Generative Model}

\icmlsetsymbol{equal}{*}

\begin{icmlauthorlist}
\icmlauthor{Tor Lattimore}{dm}
\icmlauthor{Csaba Szepesv\'ari}{dme,ua}
\icmlauthor{Gell\'ert Weisz}{dm}
\end{icmlauthorlist}
\icmlaffiliation{dm}{DeepMind, London}
\icmlaffiliation{dme}{DeepMind, Edmonton}
\icmlaffiliation{ua}{University of Alberta}
\icmlcorrespondingauthor{Tor Lattimore}{lattimore@google.com}

\icmlkeywords{linear features; bandits; reinforcement learning; misspecified models}

\vskip 0.3in
]

\begin{abstract}
The construction by \citet{du2019good} implies that even if a learner is given linear features in $\R^d$ that approximate the rewards in a bandit with a uniform error of $\epsilon$,
then searching for an action that is optimal up to $O(\epsilon)$ requires examining essentially all actions. We use the Kiefer--Wolfowitz 
theorem to prove a positive result that by checking only a few actions, a learner can always find an action that is suboptimal with an error of at most $O(\epsilon \sqrt{d})$. Thus, features are useful when the approximation error is small 
relative to the dimensionality of the features. The idea is applied to stochastic bandits and reinforcement learning with a generative model where the learner 
has access to $d$-dimensional linear features that approximate the action-value functions for all policies to an accuracy of $\epsilon$. For linear bandits, we 
prove a bound on the regret of order $\sqrt{dn \log(k)} + \epsilon n \sqrt{d} \log(n)$ with $k$ the number of actions and $n$ the horizon. For RL we show that approximate policy 
iteration can learn a policy that is optimal up to an additive error of order $\epsilon \sqrt{d}/(1 - \gamma)^2$ and using $d/(\epsilon^2(1 - \gamma)^4)$ samples from a generative model. 
These bounds are independent of the finer details of the features. We also investigate how the structure of the feature set impacts the tradeoff between sample complexity and estimation error.
\end{abstract}

\printAffiliationsAndNotice{}  

%%%%%%%%%%%%%%%%%%%%%%%%%%%%%%%%%%%%%%%%%%%%%%%%%%%%%%
% INTRODUCTION
%%%%%%%%%%%%%%%%%%%%%%%%%%%%%%%%%%%%%%%%%%%%%%%%%%%%%%
\section{Introduction}
\citet{du2019good} ask whether ``good feature representation'' is sufficient for efficient reinforcement learning and suggest a negative answer.
Efficiency here means learning a good policy with a small number of interactions either with the environment (on-line learning), 
or with a simulator (planning). A linear feature representation is called ``good'' if it approximates the value functions of \emph{all} policies with a small uniform error.
The same question can also be asked for learning in bandits.
The ideas by \citet{du2019good} suggest that the answer is also negative in finite-armed bandits with a misspecified linear model.
All is not lost, however. By relaxing the objective, we will show that
one can obtain positive results showing that efficient learning \textit{is} possible in interactive settings with good feature representations.

The rest of this article is organised as follows.
First we introduce the problem of learning to identify a near-optimal action with side information about the possible reward (\cref{sec:prob}).
We then adapt the argument of \citet{du2019good} to show that
no algorithm can find an $O(\epsapp)$-optimal action without examining nearly all the actions, even when the rewards lie within an $\epsapp$-vicinity of 
a subspace spanned by some features available to the algorithm (\cref{sec:negative}).
The negative result is complemented by a positive result showing
that there exists an algorithm such that for any feature map of dimension $d$,
the algorithm is able to find an action with suboptimality gap of at most $O(\epsapp \sqrt{d})$ where $\epsapp$ is the maximum distance between
the reward and the subspace spanned by the features in the max-norm. The algorithm only needs to investigate the reward at 
$O(d \log\log d)$ well-chosen actions. 
The main idea is to use the Kiefer-Wolfowitz theorem with a least-squares estimator of the reward function.
Finally, we apply the idea to stochastic bandits (\cref{sec:bandits}) and reinforcement learning with a generative model (\cref{sec:rl}).

\paragraph{Related work}
Despite its importance, the problem of identifying near-optimal actions when rewards follow a misspecified linear model has only recently received attention.
Of course, there is the recent paper by \citet{du2019good}, whose negative result inspired this work and is summarised for our setting in Section~\ref{sec:negative}.
A contemporaneous work also addressing the issues raised by \citet{du2019good} is by \citet{DV19}, who make a connection to the Eluder dimension \citep{RR13} and prove
a variation on our Proposition~\ref{prop:upper}. 
The setting studied here in Section~\ref{sec:negative} is closely related to the query complexity of exactly maximising a function in a given function class, which was studied by \citet{AKS11}.  
They introduced the haystack dimension as a hardness measure for exact maximisation. Unfortunately, their results for infinite classes
are generally not tight and no results for misspecified linear models were provided.
Another related area is pure exploration in bandits, which was popularised in the machine learning community by \citet{EMM06,AB10}.
The standard problem is to identify a (near)-optimal action in an unstructured bandit. \citet{SLM14} study pure exploration in linear bandits, but
do not address the case where the model is misspecified. More general structured settings have also been considered by \citet{DKM19} and others.
The algorithms in these works begin by playing every action once, which is inconsequential in an asymptotic sense.
Our focus, however, is on the finite-time regime where the number of actions is very large, which makes these algorithms unusable.
We discuss the related work on linear bandits and RL in the relevant sections later.

\paragraph{Notation}
Given a matrix $A \in \R^{n \times m}$, the set of rows is denoted by $\operatorname{rows}(A)$ and its range is  
$\Range(A) = \{A \theta : \theta \in \R^m\}$. When $A$ is positive semi-definite, we define $\norm{x}_A^2 = x^\top A x$.
The Minkowski sum of sets $U, V \subset \R^d$ is $U + V = \{u + v : u \in U, v \in V\}$.
The standard basis vectors are $e_1,\ldots,e_d$. There will never be ambiguity about deducing the dimension.

\section{Problem setup}\label{sec:prob}
We start with an abstract problem that is reminiscent of pure exploration in bandits, but without noise. Fix $\epssub>0$ and
consider the problem of identifying a $\epssub$-optimal action out of $k$ actions with the additional information that the unknown reward vector $\mu\in \R^k$ 
belongs to a known hypothesis set $\cH\subset \R^k$. An action $j\in [k] = \{1,\ldots,k\}$ is $\epssub$-optimal for $\mu = (\mu_i)_{i=1}^k$ if $\mu_j > \max_i \mu_i - \epssub$.
The learner sequentially queries actions $i \in [k]$ and observes the reward $\mu_i$. At some point the learner should stop and output both an estimated optimal action $\hat a \in [k]$ and an estimation
vector $\hat \mu \in \R^k$. There is no noise, so the learner has no reason to query the same action twice. Of course, if the learner queries all the actions, then it knows both $\mu$ and the optimal action.
The learner is permitted to randomise.
Two objectives are considered. The first only measures the quality of the outputted action $\hat a$, while the second depends on $\hat \mu$.

\begin{definition}
A learner is called sound for $(\cH, \epssub)$ if $\norm{\hat \mu - \mu}_\infty < \epssub$ almost surely for all $\mu \in \cH$.
It is called max-sound for $(\cH,\epssub)$ if $\mu_{\hat a} > \max_a \mu_a - \epssub$ almost surely for all $\mu \in \cH$.
\end{definition}

Denote by $q_\epssub(\sA,\mu)$ the expected number of queries learner $\sA$ executes when interacting with the environment specified by $\mu$ and let
\begin{align*}
    c^{\max}_\epssub(\cH) &= \inf_{\sA : \sA \text{is } (\cH, \epssub)\text{-max-sound}} \sup_{\mu \in \cH} q_\epssub(\sA,\mu) \\
    c^{\est}_\epssub(\cH) &= \inf_{\sA : \sA \text{is } (\cH, \epssub)\text{-sound}} \sup_{\mu \in \cH} q_\epssub(\sA, \mu) \,,
\end{align*}
which are the minimax query complexities for max-sound/sound learners respectively when interacting with reward vectors in $\cH$ and with error tolerance $\epssub$.
Both complexity measures are increasing as the hypothesis class is larger in the sense that
if $U\subset V$, then $c^{\max}_\epssub(U)\le c^{\max}_\epssub(V)$, and similarly for $c^{\est}_\delta$.
If a learner is sound for $(\cH, \epssub)$ and $\hat a = \argmax \hat \mu$, then clearly it is also max-sound for $(\cH, 2\epssub)$, which
shows that 
\begin{align}
c^{\max}_{2\delta}(\cH) \leq c^{\est}_{\delta}(\cH)\,. 
\label{eq:rel}
\end{align}
Our primary interest is to understand $c^{\max}_\epssub(\cH)$. Upper bounds, however, will be proven using \cref{eq:rel} and by controlling $c^{\est}_\epssub(\cH)$. Furthermore,
in Section~\ref{sec:pos} we provide a simple characterisation of $c^{\est}_\epssub(\cH)$, while $c^{\max}_\epssub(\cH)$ is apparently more subtle.
Later we need the following intuitive result, which 
says that the complexity of finding a near-optimal action when the hypothesis set consists of the unit vectors is linear in the number of actions. The 
proof is given in 
\ifsup
Section~\ref{sec:ellb}.
\else
the supplementary material.
\fi

\begin{lemma}\label{lem:obv}
$c^{\max}_1( \{ e_1, \dots, e_k \} ) = (k+1)/2$.
\end{lemma}

It follows from the aforementioned monotonicity that if $\{e_1,\ldots,e_k\} \subseteq \cH$, then $c_1^{\max}(\cH) \geq (k+1)/2$.

%%%%%%%%%%%%%%%%%%%%%%%%%%%%%%%%%%%%%%%%%%%%%%%%%%%%%%
% NEGATIVE RESULTS
%%%%%%%%%%%%%%%%%%%%%%%%%%%%%%%%%%%%%%%%%%%%%%%%%%%%%%
\section{Negative result}\label{sec:negative}
Let $\Phi \in \R^{k\times d}$.
The rows of $\Phi$ should be thought of as feature vectors assigned to each of the $k$ actions; accordingly we call $\Phi$ a feature matrix.
Furthermore, when $\mu \in \R^k$ and $a \in \rows(\Phi)$, we abuse notation by writing $\mu_a$ for the value of vector $\mu$ at the index of row $a$ in $\Phi$.
Our interest lies in the regime where $k$ is much larger than $d$ and where $\exp(d)$ is large.

We consider hypothesis classes where the true reward lies within an $\epsapp$-vicinity of $\Range(\Phi)$ as measured in max-norm. 
Let $\cH_\Phi^\epsapp = \Range(\Phi) + B_\infty(\epsapp)$, where $B_\infty(\epsapp) = [-\epsapp, \epsapp]^k$ is a $k$-dimensional hypercube.
How large is $c_\epssub^{\max}(\cH_\Phi^\epsapp)$ for different regimes of $\epssub$ and $\epsapp$ and feature matrices?
As we shall see, for $\epssub=\Omega(\epsapp \sqrt{d})$ one can keep the complexity small, while for smaller $\epssub$ there exist feature matrices for which 
the complexity can be as high as the large dimension, $k$.

The latter result follows from the core argument of the recent paper by \citet{du2019good}.
The next lemma is the key tool, and is a consequence of the Johnson--Lindenstrauss lemma. 
It shows that there exist matrices $\Phi \in \R^{k \times d}$ with $k$ much larger than $d$ where
rows have unit length and all non-equal rows are almost orthogonal.

\begin{lemma}
\label{lem:jl}
For any $\epsapp>0$ and $d\in [k]$ such that 
$d\ge \lceil 8 \log(k)/\epsapp^2 \rceil$,
there exists a feature matrix $\Phi\in \R^{k\times d}$ with unique rows such that 
for all $a, b \in \rows(\Phi)$ with $a \neq b$, $\norm{a}_2 = 1$ and $|a^\top b| \leq \epsapp$.
\end{lemma}

\cref{lem:obv,lem:jl} together imply the promised result:

\begin{proposition}\label{prop:badmx}
For any $\epsapp>0$ and $d\in [k]$ with 
$d\ge \lceil 8 \log(k)/\epsapp^2 \rceil$,
there exists a feature matrix $\Phi\in \R^{k\times d}$ such that $c_1^{\max}(\cH_\Phi^\epsapp) \geq (k+1)/2$. 
\end{proposition}

\begin{proof}
Let $\Phi$ be the matrix from \cref{lem:jl} with $\rows(\Phi) = (a_i)_{i=1}^k$.
We want to show that $e_i \in \cH_\Phi^\epsapp$ for all $i \in [k]$ and then apply \cref{lem:obv}. 
If $\theta = a_i$, then $\Phi \theta = ( a_1^\top a_i, \dots, a_i^\top a_i, \dots, a_k^\top a_i)^\top$.
By the choice of $\Phi$ the $i$th component is one and the others are all less than $\epsapp$ in absolute value.
Hence, $\|\Phi \theta - e_i\|_\infty\le \epsapp$, which completes the proof.
\end{proof}

The proposition has a worst-case flavour. Not all feature matrices have a high query complexity. 
For a silly example, the query complexity of the zero matrix $\Phi=\zeros$ satisfies $c_1^{\max}(\cH_\Phi^\epsapp) = 0$ provided $\epsapp < 1$.
That said, the matrix witnessing the claims in \cref{lem:jl} can be found with non-negligible probability by sampling the rows of $\Phi$ uniformly from the surface of
the $(d-1)$-dimensional sphere.
There is another way of writing this result, emphasising the role of the dimension rather than the number of actions.

\begin{corollary}\label{cor:badmx}
For all $\delta > \epsilon$, there exists a feature matrix $\Phi \in \R^{k \times d}$ with suitably large $k$ such that
\begin{align*}
c_\delta^{\max}(\cH^\epsilon_\Phi) \geq \frac{1}{2} \exp\left(\frac{d-1}{8} \left(\frac{\epsilon}{\delta}\right)^2\right)\,.
\end{align*}
\end{corollary}

The proof follows by rescaling the features in \cref{prop:badmx} and is given in 
\ifsup
\cref{app:cor:badmx}.
\else
the supplementary material.
\fi

%%%%%%%%%%%%%%%%%%%%%%%%%%%%%%%%%%%%%%%%%%%%%%%%%%%%%%
% POSITIVE RESULTS
%%%%%%%%%%%%%%%%%%%%%%%%%%%%%%%%%%%%%%%%%%%%%%%%%%%%%%
\section{Positive result}\label{sec:pos}
The negative result of the previous section is complemented with a positive result showing that the query 
complexity can be bounded independently of $k$ whenever $\delta = \Omega(\epsilon \sqrt{d})$.
For the remainder of the article we make the following assumption:

\begin{assumption}
$\Phi \in \R^{k \times d}$ has unique rows and the span of $\rows(\Phi)$ is all of $\R^d$.
\end{assumption}

We discuss the relationship between this result and \cref{prop:badmx} at the end of the section.

\begin{proposition}\label{prop:upper}
Let $\Phi \in \R^{k \times d}$ and
$\epssub > 2\epsapp(1 + \sqrt{2d})$. Then, $c^{\max}_\epssub(\cH_\Phi^\epsapp) \leq 4 d \log \log d + 16$.
\end{proposition}

The proof relies on the Kiefer--Wolfowitz theorem, which we now recall.
Given a probability distribution $\rho : \rows(\Phi) \to [0,1]$, let $G(\rho) \in \R^{d\times d}$ and $g(\rho) \in \R$ be given by
\begin{align*}
G(\rho) &= \sum_{a \in \rows(\Phi)} \rho(a) a a^\top\,, & 
g(\rho) &= \max_{a\in \rows(\Phi)} \norm{a}_{G(\rho)^{-1}}^2 \,.
\end{align*}

\begin{theorem}[\citealt{KW60}]\label{thm:kiefer-wolfowitz} 
The following are equivalent: 
\begin{enumerate}[itemsep=0pt,nosep]
\item $\rho^*$ is a minimiser of $g$. 
\item $\rho^*$ is a maximiser of $f(\rho) = \log \det G(\rho)$. \label{thm:des:kw:2}
\item $g(\rho^*) = d$.  \label{thm:des:kw:3}
\end{enumerate}
Furthermore, there exists a minimiser $\rho^*$ of $g$ such that the support of $\rho^*$ has cardinality at most $|\Supp(\rho)| \leq d(d+1)/2$.
\end{theorem} 

The distribution $\rho^*$ is called an (optimal) experimental design and
the elements of its support are called its core set.
Intuitively, when covariates are sampled from $\rho$, then $g(\rho)$ 
is proportional to the maximum variance of the corresponding least-squares estimator 
over all directions in $\rows(\Phi)$. Hence, minimising $g$ corresponds to minimising the worst-case variance of the resulting least-squares estimator.
A geometric interpretation is that the core set lies on the boundary of the central ellipsoid of minimum volume that contains $\rows(\Phi)$.
The next theorem shows that there exists a near-optimal design with a small core set.
The proof follows immediately from part (ii) of lemma 3.9 in the book by \citet{Tod16},
which also provides an algorithm for computing such a distribution in roughly order $k d^2$ computation steps. 

\begin{theorem}\label{thm:todd}
There exists a probability distribution $\rho$ such that $g(\rho) \leq 2d$ and the core set of $\rho$ has size at most $4d \log \log(d) + 16$.
\end{theorem}

The proof of \cref{prop:upper} is a corollary of the following more general result about least-squares estimators 
over near-optimal designs.

\begin{proposition}\label{prop:inf}
Let $\mu \in \cH_\Phi^\epsilon$ and $\eta \in [-\beta, \beta]^k$.
Suppose that $\rho$ is a probability distribution over $\rows(\Phi)$ satisfying the conclusions of \cref{thm:todd}. 
Then $\norm{\Phi \hat \theta - \mu}_\infty \leq \epsilon + (\epsilon + \beta) \sqrt{2d}$, where
\begin{align*}
\hat \theta = G(\rho)^{-1} \sum_{a \in \rows(\Phi)} \rho(a) (\mu_a + \eta_a) a \,.
\end{align*}
\end{proposition}

The problem can be reduced to the case where $\eta = \zeros$ by noting that $\mu + \eta \in \cH_\Phi^{\epsilon+\beta}$
The only disadvantage is that this leads to an additional additive dependence on $\beta$.

\begin{proof}
Let $\mu = \Phi \theta + \Delta$ where $\norm{\Delta}_\infty \leq \epsilon$.
The difference between $\hat \theta$ and $\theta$ can be written as
\begin{align*}
\hat \theta - \theta 
&= G(\rho)^{-1} \sum_{a \in \rows(\Phi)} \rho(a) a\left(a^\top \theta + \Delta_a + \eta_a\right) - \theta \\
&= G(\rho)^{-1} \sum_{a \in \rows(\Phi)} \rho(a) (\Delta_a + \eta_a) a\,.
\end{align*}
Next, for any $b \in \rows(\Phi)$,
\begin{align*}
&\ip{b, \hat \theta - \theta} 
= \bip{b, G(\rho)^{-1} \sum_{a \in \rows(\Phi)} \rho(a) (\Delta_a+\eta_a) a} \nonumber \\
&= \sum_{a \in \rows(\Phi)} \rho(a) (\Delta_a + \eta_a) \ip{b, G(\rho)^{-1} a} \nonumber \\
&\leq (\epsapp + \beta) \sum_{a \in \rows(\Phi)} \rho(a) |\ip{b, G(\rho)^{-1} a}| \nonumber \\ 
&\leq (\epsapp + \beta) \sqrt{\sum_{a \in \rows(\Phi)} \rho(a) \ip{b, G(\rho)^{-1} a}^2} \nonumber \\
&= (\epsapp + \beta) \sqrt{\sum_{a \in \rows(\Phi)} \rho(a) b^\top G(\rho)^{-1} a a^\top G(\rho)^{-1} b}  \\
&= (\epsapp + \beta) \sqrt{\norm{b}^2_{G(\rho)^{-1}}} 
\leq (\epsapp + \beta) \sqrt{g(\rho)} \leq (\epsapp + \beta) \sqrt{2d}\,, 
\end{align*}
where the first inequality follows from H\"older's inequality and the 
fact that $\norm{\Delta}_\infty \leq \epsapp$, the second by Jensen's inequality and the last two by our choice of $\rho$ and \cref{thm:todd}.
Therefore
\begin{align*}
\ip{b, \hat \theta} \leq \ip{b, \theta} + (\epsapp+\beta) \sqrt{2d} \leq \mu_b + \epsapp + (\epsapp + \beta) \sqrt{2d}\,.
\end{align*}
A symmetrical argument completes the proof.
\end{proof}

\begin{proof}[Proof of \cref{prop:upper}]
Let $\rho$ be a probability distribution over $\rows(\Phi)$ satisfying the conclusions of \cref{thm:todd}.
Consider the algorithm that evaluates $\mu$ on each point of the support of $\rho$ and computes the least-squares
estimator defined in \cref{prop:inf} and predicts $\hat a = \argmax_{a \in \rows(\Phi)} \ip{a, \hat \theta}$.
Let $a^*=\argmax_{a\in \rows(\Phi)} \mu_a$ be the optimal action. 
Then by \cref{prop:inf} with $\eta = \zeros$,
\begin{align*}
\mu_{\hat a} 
&\geq \ip{\hat a, \hat \theta} - \epsapp\left(1 + \sqrt{2d}\right) 
\geq \ip{a^*, \hat \theta} - \epsapp\left(1 + \sqrt{2d}\right) \\
&\geq \mu_{a^*} - 2\epsapp \left(1 + \sqrt{2d}\right) > \mu_{a^*} - \epssub\,. 
\qedhere
\end{align*}
\end{proof}

\paragraph{Discussion}
\cref{cor:badmx} shows that the query complexity is exponential in $d$ when
$\epssub$ is not much larger than $\epsapp$, but is benign when $\epssub = \Omega(\epsapp \sqrt{d})$.
The positive result shows that in the latter regime the complexity is more or less linear in $d$. Precisely,
\begin{align*}
\min\left\{\epssub : c^{\max}_\epssub(\cH_\Phi^\epsapp) \leq 4 d \log \log(d) + 16 \right\} = O(\epsapp \sqrt{d})\,.
\end{align*}

The message is that there is a sharp tradeoff between query complexity and error. The learner pays dearly in terms of query complexity if they demand an estimation error that is close to the approximation error.
By sacrificing a factor of $\sqrt{d}$ in estimation error, the query complexity is practically just linear in $d$.

\paragraph{Comparison to supervised learning}
As noted by \citet{du2019good}, the negative result does not hold in supervised learning, where the learner is judged on its average prediction error
with respect to the data generating distribution.
Suppose that $a,a_1,\ldots,a_n$ are sampled i.i.d.\ from some distribution $P$ on $\rows(\Phi)$ and the learner observes $(a_t)_{t=1}^n$ and $(\mu_{a_t})_{t=1}^n$. 
\begin{align*}
\hat \theta = \left(\sum_{t=1}^n a_t a_t^\top\right)^{-1} \sum_{t=1}^n a_t \mu_{a_t}\,.
\end{align*}
Then, by making reasonable boundedness and span assumptions on $\rows(\Phi)$, and by combining the results in chapters 13 and 14 of \citep{Wai19}, with high probability,
\begin{align*}
\E\left[(a^\top \hat \theta - \mu_a)^2 \,\Big|\, \hat \theta\right] = O\left(\frac{d}{n} + \epsilon^2\right)\,. 
\end{align*}
Notice, there is no $d$ multiplying the dependence on the approximation error.
The fundamental difference is that $a$ is sampled from $P$. The quantity $\max_{a \in \rows(\Phi)} (a^\top \hat \theta - \mu_a)^2$ behaves quite differently, as the lower bound shows.

\paragraph{Feature-dependent bounds}
The negative result in Section~\ref{sec:negative} shows that there \textit{exist} feature matrices for which the
learner must query exponentially many actions or suffer an estimation error that expands the approximation error by a factor of $\sqrt{d}$.
On the other hand, Proposition~\ref{prop:upper} shows that for \textit{any} feature matrix, there exists a learner that queries $O(d \log \log(d))$ 
actions for an estimation error of $\epsilon \sqrt{d}$, roughly matching the lower bound.
One might wonder whether or not there exists a feature-dependent measure that characterises the blowup in estimation error in terms of the feature matrix and
query budget. One such measure is given here. Given a set $C \subseteq [k]$ with $|C| = q$, let $\Phi_C \in \R^{q \times d}$ be the matrix obtained from $\Phi$
by restricting to those rows indexed by $C$. Define
\begin{align*}
\lambda_q(\Phi) 
&= \min_{\substack{C \subset [k], |C| = q}} \max_{v \in \R^d \setminus \{\zeros\}} \frac{\norm{\Phi v}_\infty}{\norm{\Phi_C v}_\infty}\,.
\end{align*}

\begin{proposition}\label{prop:bounds}
Let $1 \leq q < k$ and 
$\delta_1 = \epsilon(1 + \lambda_q(\Phi))$ and $\delta_2 > \epsilon(1 + 2\lambda_q(\Phi))$.
Then, 
\begin{align*}
c^{\est}_{\delta_1}(\cH^\epsilon_\Phi) > q \geq c^{\est}_{\delta_2}(\cH^\epsilon_\Phi)\,.
\end{align*}
\end{proposition}

\ifsup
The proof is supplied in Appendix~\ref{app:bounds}.
\else
The proof is supplied in the supplementary material.
\fi
By (\ref{eq:rel}), it also holds that $c^{\max}_{2\delta_2}(\cH^\epsilon_\Phi) \leq q$. Currently we do not have a corresponding lower bound, however.

%%%%%%%%%%%%%%%%%%%%%%%%%%%%%%%%%%%%%%%%%%%%%%%%%%%%%%
% MISSPECIFIED LINEAR BANDITS
%%%%%%%%%%%%%%%%%%%%%%%%%%%%%%%%%%%%%%%%%%%%%%%%%%%%%%
\section{Misspecified linear bandits}\label{sec:bandits}

Here we consider the classic stochastic bandit where the mean rewards are nearly a linear function of their associated features. 
We assume for simplicity that no two actions have the same features. In case this does not hold, a representative action can be chosen for each feature without
changing the main theorem.
Let $\Phi \in \R^{k \times d}$ and $\mu \in \cH_\Phi^\epsapp$. 
In rounds $t \in [n]$, the learner chooses actions $(X_t)_{t=1}^n$ with $X_t \in \rows(\Phi)$ and the reward is $Y_t = \mu_{X_t} + \eta_t$ 
where $(\eta_t)_{t=1}^n$ is a sequence of independent $1$-subgaussian random variables.
The optimal action has expected reward $\mu^* = \max_{a \in \cA} \mu_a$ and the
expected regret is $R_n = \E[\sum_{t=1}^n \mu^* - \mu_{X_t}]$.
The idea is to use essentially the same elimination algorithm as \citep[chapter~22]{LaSz18:book}, which is summarised in \cref{alg:elim}.
In each episode, the algorithm computes a near-optimal design over a subset of the actions that are plausibly optimal.
It then chooses each action in proportion to the optimal design and eliminates arms that appear sufficiently suboptimal. 

\begin{proposition}\label{prop:linear}
When $\alpha = 1/(kn)$ and $C$ is a suitably large universal constant 
and $\max_a \mu_a - \min_a \mu_a\le 1$,
\cref{alg:elim} satisfies
\begin{align*}
    R_n \leq C\left[\sqrt{d n \log(nk)} + \epsapp n \sqrt{d} \log(n)\right]\,.
\end{align*}
\end{proposition}

\ifsup
In Appendix~\ref{sec:bandit-lower}, 
\else
In the supplementary material,
\fi
we show that the bound in Proposition~\ref{prop:linear} is tight up to logarithmic factors in the interesting regime where $k$ is comparable to $n$.

\begin{proof}
Let $\mu = \Phi \theta + \Delta$ with $\norm{\Delta}_\infty \leq \epsapp$, which exists by the assumption that $\mu \in \cH_\Phi^\epsapp$.
We only analyse the behaviour of the algorithm within an episode, showing that the least-squares estimator is guaranteed
to have sufficient accuracy so that (a) arms that are sufficiently suboptimal are eliminated and (b) some near-optimal arms are retained. Fix any $b\in \cA$.
Using the notation in \cref{alg:elim},
\begin{align}
&\ip{b, \hat \theta - \theta} 
= \left|b^\top G^{-1} \sum_{s=1}^u \Delta_{X_s} X_s + b^\top G^{-1} \sum_{s=1}^u X_s \eta_s\right| \nonumber \\
&\,\,\leq \left|b^\top G^{-1} \sum_{a \in \cA} u(a) a \Delta_a \right| + \left|b^\top G^{-1} \sum_{s=1}^u X_s \eta_s\right|\,.
\label{eq:bandit1}
\end{align}
The first term is bounded using Jensen's inequality as before:
\begin{align*}
&\left|b^\top G^{-1} \sum_{a \in \cA} u(a) \Delta_a a\right|
\leq \epsapp \sum_{a \in \cA} u(a) \left|b^\top G^{-1} a\right| \\
&\quad\leq \epsapp \sqrt{\left(\sum_{a \in \cA} u(a)\right) b^\top \sum_{a \in \cA} u(a) G^{-1} aa^\top G^{-1} b} \\
&\quad= \epsapp \sqrt{\sum_{a \in \cA} u(a) \norm{b}^2_{G^{-1}}} 
\leq \epsapp \sqrt{\frac{2d u}{m}}
\leq 2\epsapp \sqrt{d}\,,
\end{align*}
where the first inequality follows form H\"older's inequality,
the second is Jensen's inequality and the last follows from the exploration distribution that guarantees $\norm{b}^2_{G^{-1}} \leq 2d/m$.
The second term in \cref{eq:bandit1} is bounded using standard concentration bounds. Preciesly, by eq. (20.2) of \cite{LaSz18:book}, with probability at least $1 - 2\alpha$,
\begin{align*}
\left|b^\top G^{-1} \sum_{s=1}^u X_s \eta_s\right| 
&\leq \norm{b}_{G^{-1}} \sqrt{2 \log\left(\frac{1}{\alpha}\right)} \\
&\leq \sqrt{\frac{4d}{m} \log\left(\frac{1}{\alpha}\right)}
\end{align*}
and $|\ip{b, \hat \theta - \theta} |\le  2\epsapp \sqrt{d} + 
\sqrt{\frac{4d}{m} \log\left(\frac{1}{\alpha}\right)}$.
Continuing with standard calculations, provided in
\ifsup
\cref{app:bandit},
\else
the supplementary material
\fi
one gets that the expected regret satisfies 
\begin{align*}
R_n \leq C\left[\sqrt{d n \log(nk)} + \epsapp n \sqrt{d} \log(n)\right] 
\end{align*}
where $C > 0$ is a suitably large universal constant. The logarithmic factor in the second term is due to the fact that in each of the logarithmically many episodes the algorithm may eliminate the best remaining arm, but
keep an arm that is at most $O(\epsapp \sqrt{d})$ worse than the best remaining arm.
\end{proof}

\newcommand{\algitem}[1]{\item[(#1)]}

\begin{algorithm}
\textsc{input} $\Phi \in \R^{k \times d}$ and confidence level $\alpha \in (0,1)$
\begin{enumerate}[leftmargin=0.6cm]
\algitem{1} Set $m = \ceil{4 d \log \log d} + 16$ and $\cA = \rows(\Phi)$
\algitem{2} Find design $\rho : \cA \to [0,1]$ with $g(\rho) \leq 2d$ and $|\Supp(\rho)| \leq 4d \log \log(d) + 16$
\algitem{3}  Compute $\displaystyle u(a) = \ceil{m\rho(a)}$ and $\displaystyle u = \sum_{a \in \cA} u(a)$
\algitem{4} Take each action $a \in \cA$ exactly $u(a)$ times with corresponding features $(X_s)_{s=1}^u$ and rewards $(Y_s)_{s=1}^u$
\algitem{5} Calculate the vector $\hat \theta$: 
\begin{align*}
\hat \theta = G^{-1} \sum_{s=1}^u X_s Y_s \quad \text{with} \quad G = \sum_{a \in \cA} u(a) aa^\top
\end{align*}
\algitem{6} Update active set: \\ \!\!\!$\displaystyle \cA \leftarrow \left\{a \in \cA : \max_{b \in \cA} \ip{\hat \theta, b - a} \leq 2\sqrt{\frac{4d}{m} \log\left(\frac{1}{\alpha}\right)}\right\}$.
\algitem{7} $m \leftarrow 2m$ and \textsc{goto} (1) 
\end{enumerate}
\caption{\textsc{phased elimination}}\label{alg:elim}
\end{algorithm}

\begin{remark}
When the active set contains fewer than $d$ actions, then the conditions of Kiefer-Wolfowitz are not satisfied because $\cA$ cannot span $\R^d$.
Rest assured, however, since in these cases one can simply work in the smaller space spanned by $\cA$ and the analysis goes through without further changes.
\end{remark}

\paragraph{Known approximation error}
The logarithmic factor in the second term in the regret bound can be removed when $\epsapp$ is known by modifying the elimination criteria so that with high probability the
optimal action is never eliminated, as explained in
\ifsup
\cref{rem:known}.
\else
the supplementary material.
\fi

\paragraph{Infinite action sets}
The logarithmic dependence on $k$ follows from the choice of $\alpha$, which is needed to guarantee the concentration holds for all actions.
When $k = \Omega(\exp(d))$, the union bound can be improved by a covering argument or using the argument in the next section. This leads to a bound
of $O(d \sqrt{n \log(n)} + \epsilon n \sqrt{d} \log(n))$, which is independent of the number of arms.

\paragraph{Other approaches}
We are not the first to consider misspecified linear bandits.
\citet{ghchgo07} consider the same setting and show that in the favourable case when one can cheaply test linearity,  there exist algorithms for which the regret has
order $\min(d, \sqrt{k}) \sqrt{n}$ up to logarithmic factors. While such results a certainly welcome, our focus is on the case where $k$ has the same order of magnitude as $n$ and hence
the dependence of the regret on $\epsilon$ is paramount.
Another way to obtain a similar result to ours is to use the Eluder dimension \citep{RR13}, which should first be generalised a little to accommodate the need to use an accuracy
threshhold that does not decrease with the horizon. Then the Eluder dimension can be controlled using either our techniques or the alternative argument by \citet{DV19}. 

\paragraph{Contextual linear bandits} 
Algorithms based on phased elimination are not easily adapted to the contextual case, which is usually addressed using optimistic methods.
You might wonder whether or not LinUCB \citep{AST11} serendipitously adapts to misspecified models in the contextual case.
\citet{GMZ16} have shown that LinUCB \textit{is} robust in the non-contextual case when $\epsilon$ is very small.
Their conditions, however, depend the structure of the problem, and in particular on having good control of the $2$-norm of $\Delta$, which may 
scale like $\Omega(\epsilon \sqrt{k})$ and is too big for large action sets.
We provide a negative result in the supplementary material, as well as a modification that corrects the algorithm, but requires knowledge of the approximation error.
The modification is a data-dependent refinement of the bonus used by \citet{JYW19}.
An open question is to find an algorithm for contextual linear bandits for which the regret similar to \cref{prop:linear} and where the algorithm does not
need to know the approximation error.

%%%%%%%%%%%%%%%%%%%%%%%%%%%%%%%%%%%%%%%%%%%%%%%%%%%%%%
% REINFORCEMENT LEARNING
%%%%%%%%%%%%%%%%%%%%%%%%%%%%%%%%%%%%%%%%%%%%%%%%%%%%%%
\section{Reinforcement learning}\label{sec:rl}
We now consider discounted reinforcement learning with a generative model, 
which means the learner can sample next-states and rewards for any state-action pair of their choice.
The notation is largely borrowed from \citep{Sze10}.
Fix an MDP with state space $[S]$, action space $[A]$, transition kernel $P$, reward function $r : [S] \times [A] \to [0,1]$ 
and discount factor $\gamma \in (0, 1)$. The finiteness of the state space is assumed only for simplicity.
As usual, $V^\pi$ and $Q^\pi$ refer to the value and action-value functions for policy $\pi$ (e.g., $V^\pi(s)$ is the total expected discounted reward incurred while following policy $\pi$ in the MDP) and $V^*$ and $Q^*$ the same for the optimal policy.
The learner is given a feature matrix $\Phi \in \R^{SA \times d}$ such that $Q^\pi \in \cH_\Phi^\epsapp$ for all policies $\pi$ and where $Q^\pi$ is vectorised
in the obvious way. The notation $\Phi(s, a) \in \R^d$ denotes the feature associated with state-action pair $(s, a)$.

The main idea is the observation that if $Q^*$ were known with reasonable accuracy on the support
of an approximately optimal design $\rho$ on the set of vectors $(\Phi(s, a) : s,a \in [S] \times [A])$, then least-squares in combination with our earlier arguments
would provide a good estimation of the optimal state-action value function.
Approximating $Q^*$ on the core set $\cC = \Supp(\rho) \subset [S] \times [A]$ is possible using approximate policy iteration.
For the remainder of this section let $\rho$ be a design with $g(\rho) \leq 2d$ and with support $\cC$ and
$G(\rho) = \sum_{(s, a) \in \cC} \rho(s, a) \Phi(s, a) \Phi(s,a)^\top$.

\paragraph{Related work}
The idea to extrapolate a value function by sampling from a few anchor state/action pairs has been used before in a few works. 
The recent work by \citet{ZLK19} consider approximate value iteration in the episodic
setting and do not make a connection to optimal design. 
The challenge in the finite-horizon setting is that one must learn one parameter vector for each layer and, at least naively, errors propagate multiplicatively. 
For this reason using the anchor pairs from the support of an experimental design would not make the algorithm proposed by the aforementioned paper practical.
\citet{YW19} assume the transition matrix has a linear structure and also use least-squares with data from a pre-selected collection of anchor state/action pairs.
Their assumption is that the features of all state-action pairs can be written as a convex combination of the anchoring features, which means the number of anchors
is the number of corners of the polytope spanned by $\rows(\Phi)$ and may be much larger than $d$. One special feature of their paper is that the dependency on the horizon of the sample
complexity is cubic in $1/(1 - \gamma)$, while in our theorem it is quartic.
Earlier, \citet{LaBhSze18} described how anchor states (with some lag allowed) can be used to reduce the number of constraints in the approximate linear programming approach to approximate planning in MDPs, while maintaining error bounds.

\paragraph{Approximate policy iteration}
Let $\pi_1$ be an arbitrary policy and define
a sequence of policies $(\pi_k)_{k=1}^\infty$ inductively using the following procedure.
From each state-action pair $(s, a) \in \cC$ take $m$ roll-outs of length $n$ following policy $\pi_k$ and let $\hat Q_k(s, a)$ 
be the empirical average, which is only defined on the core set $\cC$. 
The estimation of $Q^{\pi_k}$ is then extended to all state-action pairs using the features and least-squares
\begin{align*}
\hat \theta_k = G(\rho)^{-1} \sum_{(s, a) \in \cC} \rho(s, a) \Phi(s, a) \hat Q_k(s, a) \quad 
Q_k = \Phi \hat \theta_k\,.
\end{align*}
Then $\pi_{k+1}$ is chosen to be the greedy policy with respect to $Q_k$ and the process is repeated.
The following theorem shows that for suitable choices of roll-out length $n$, roll-out number $m$ and iterations $k$, the policy 
$\pi_{k+1}$ is nearly optimal with high probability. Significantly, the choice of parameters ensures that the total number of samples
from the generative model is independent of $S$ and $A$.

\begin{theorem}\label{thm:mdp}
Suppose that approximate policy iteration is run with
\begin{align*}
k = \frac{\log\left(\frac{1}{\epsapp \sqrt{d}}\right)}{1 - \gamma}  \quad
m = \frac{\log\left(\frac{2k|\cC|}{\alpha}\right)}{2\epsapp^2(1 - \gamma)^2} \quad
n = \frac{\log\left(\frac{1}{\epsapp(1-\gamma)}\right)}{1 - \gamma}\,.
\end{align*}
Then there exists a universal constant $C$ such that with probability at least $1 - \alpha$, the policy $\pi_{k+1}$ satisfies
\begin{align*}
\max_{s \in [S]} \left(V^*(s) - V^{\pi_{k+1}}(s)\right) \leq C\epsapp \sqrt{d} / (1 - \gamma)^2\,,
\end{align*}
\end{theorem}

When $\rho$ is chosen using \cref{thm:todd} so that $|\cC| \leq 4 d \log \log(d) + 16$, then
the number of samples from the generative model is $kmn|\cC|$, which is
\begin{align*}
O\left(\frac{\log\left(\frac{1}{\epsapp(1 - \gamma)}\right) \log\left(\frac{2k|\cC|}{\alpha}\right)
\log\left(\frac{1}{\epsapp \sqrt{d}}\right) d \log \log(d)}{\epsapp^2 (1 - \gamma)^4}\right) \,.
\end{align*}
Before the proof we need two lemmas.
The first controls the propagation of errors in policy iteration when using $Q_k$ rather than $Q^{\pi_k}$.
For policy $\pi$, let $P^\pi : \R^{[S] \times [A]} \to \R^{[S] \times [A]}$ defined by $(P^\pi Q)(s,a)=\sum_{s'} P(s'|s,a) Q(s',\pi(s'))$.

\begin{lemma}\label{lem:mdp1}
Let $\epssub_i = Q_i - Q^{\pi_i}$ and $E_i = P^{\pi_{i+1}}(I-\gamma P^{\pi_{i+1}})^{-1} (I-\gamma P^{\pi_i})-P^{\pi^*}$.
Then,
$ %\begin{align*}
Q^* - Q^{\pi_k} 
%&\leq 
\le
(\gamma P^{\pi^*})^k (Q^* - Q^{\pi_0}) 
%\\
%&\qquad\qquad 
+ \gamma \sum_{i=0}^{k-1} (\gamma P^{\pi^*})^{k-i-1} E_i  \epssub_i 
$.
\end{lemma}

\begin{proof}
This is stated as Eq. (7) in the proof of part~(b) of Theorem~3 of \cite{FaMuSz10}
and ultimately follows from Lemma~4 of \citet{Munos03}.
\end{proof}

The second lemma controls the value of the greedy policy with respect to a $Q$ function in terms of the quality of the $Q$ function.

\begin{lemma}[\citet{SinghYee94}, corollary 2]\label{lem:mdp2}
Let $\pi$ be greedy with respect to an action-value function $Q$.
Then for any state $s\in [S]$, $V^\pi(s)\ge V^*(s) - \frac{2}{1-\gamma} \norm{Q-Q^*}_\infty$.
\end{lemma}

\begin{proof}[Proof of \cref{thm:mdp}]
Hoeffding's bound and the definition of the roll-out length shows that for any $(s, a) \in \cC$, with probability at least $1 - \alpha$,
\begin{align*}
\left|\hat Q_i(s, a) - Q^{\pi_i}(s,a)\right| \leq \frac{1}{1 - \gamma}\sqrt{\frac{1}{2m} \log\left(\frac{2}{\alpha}\right)} + \epsapp  = 2\epsapp\,.
\end{align*}
At the end we analyse the failure probability of the algorithm, but for now assume the above inequality holds for all $i \leq k$ and $(s, a) \in \cC$.
Let $\theta_i = \argmin_\theta \norm{Q^{\pi_i} - \Phi \theta}_{\infty}$. Then, by \cref{prop:inf} with $\beta = 2\epsilon$,
\begin{align*}
\norm{Q_i - Q^{\pi_i}}_\infty 
= \norm{\Phi \hat \theta_i - Q^{\pi_i}}_\infty
\leq 3 \epsapp \sqrt{2d} + \epsapp 
\doteq \epssub \,.
\end{align*}
Since the rewards belong to the unit interval,
taking the maximum norm of both sides in \cref{lem:mdp1} shows that
$\norm{Q^* - Q^{\pi_k}}_\infty \leq 2\epssub / (1-\gamma) + \gamma^k / (1-\gamma)$.
Then, by the triangle inequality,
\begin{align*}
\norm{Q_k-Q^*}_\infty
&\leq \norm{Q_k-Q^{\pi_k}}_\infty + \norm{Q^*-Q^{\pi_k}}_\infty \\
&\leq \frac{3\epssub}{1-\gamma} + \frac{\gamma^k}{1-\gamma}\,.
\end{align*}
Next, by \cref{lem:mdp2}, for any state $s \in [S]$, 
\begin{align*}
V^{\pi_{k+1}}(s)
& \geq V^*(s) - \frac{2}{1-\gamma} \norm{Q_k-Q^*}_\infty \\
& \geq V^*(s) - \frac{2}{(1-\gamma)^2} \left( 3\epssub + \gamma^k\right)\,. 
\end{align*}
All that remains is bounding the failure probability, which follows immediately from a union bound over all iterations $i \leq k$ and state-action pairs $(s, a) \in \cC$.
\end{proof}

%%%%%%%%%%%%%%%%%%%%%%%%%%%%%%%%%%%%%%%%%%%%%%%%%%%%%%
% CONCLUSIONS
%%%%%%%%%%%%%%%%%%%%%%%%%%%%%%%%%%%%%%%%%%%%%%%%%%%%%%
\section{Conclusions}
Are good representations sufficient for efficient learning in bandits or in RL with a generative model? 
The answer depends on whether one accepts a blowup of the approximation error by a factor of $\sqrt{d}$, and 
is positive if and only if this blowup is acceptable.
The implication is that the role of bias/prior information is more pronounced than in supervised learning where the blowup does not appear.
One may wonder whether the usual changes to the learning problem, such as considering sparse approximations, could reduce the blowup. 
Since sparsity is of little help even in the realisable setting \citep[chapter~23]{LaSz18:book}, we are only modestly optimistic in this regard.
Note also that in reinforcement learning, the blowup is even harsher: in the discounted case we see that a factor of $1/(1-\gamma)^2$ also appears, which 
we believe is not improvable.

The analysis in both the bandit and reinforcement learning settings can be decoupled into two components.
The first is to control the query complexity of identifying a near-optimal action and the second is estimating the value of an action/policy using roll-outs.
This view may be prove fruitful when analysing (more) non-linear classes of reward function.

There are many open questions.
First, in order to compute an approximate optimal design, the algorithm needs to examine all features.
Second, the argument in \cref{sec:rl} 
heavily relies on the uniform contraction property of the various operators involved. It remains to be seen whether similar arguments hold for other settings, 
such as the finite horizon setting or the average 
cost setting.
Another interesting open question is whether a similar result holds for the online setting when the learner needs to control its regret.

\bibliography{all}

\appendix

\ifsup

%%%%%%%%%%%%%%%%%%%%%%%%%%%%%%%%%%%%%%%%%%%%%%%%%%%%%%
% PROOF OF OBVIOUS LEMMA
%%%%%%%%%%%%%%%%%%%%%%%%%%%%%%%%%%%%%%%%%%%%%%%%%%%%%%
\section{Proof of \cref{lem:obv}}
\label{sec:ellb}
Recall that \cref{lem:obv} states that 
\begin{align*}
c^{\max}_1( \{ e_1, \dots, e_k \} ) = (k+1)/2\,.
\end{align*}
For the upper bound consider, an algorithm that queries each coordinate in a random order, stopping as soon as it receives a nonzero reward, at which point the algorithm can return $\hat \mu = \mu$ and $\hat a = \argmax \mu$. 
Clearly, this algorithm is sound. Let $S\in \mathrm{Perm}([k])$ be a (uniform) random permutation of $[k]$, which represents the order of the algorithms queries.
Let $T_i = \min\{ t\ge 1\,:\, S_t = i \}$ be the number of queries if $r=e_i$. Note that $(T_i)_i$ is a random permutation of $[k]$.
Note also that by symmetry $\EE{T_1}=\dots = \EE{T_i}$.
Hence, $\max_i \EE{T_i} = \frac1k \EE{\sum_i T_i} = (k+1)/2$.

The proof of the lower bound is based on a similar calculation.
First, note that by Yao's principle it suffices to show that there exist a distribution $P$ over the problem instances such that any \emph{deterministic} algorithm needs to ask at least $(k+1)/2$ queries on expectation when the instance that the algorithm runs on is chosen randomly from the distribution. Let $P$ be the uniform distribution. Note that any deterministic algorithm $\cA$ can be identified with a fixed sequence $s_1,s_2, \dots\in [k]$ of queries. 
Call $\cA$ dominated if some other algorithm $\cA'$ achieves better query complexity on any input instance while the query complexity of $\cA'$ is never worse than that of $\cA$ on any other instance.
Clearly, it suffices to consider nondominated algorithms.
Hence, $s_1,s_2,\dots,s_k$ cannot have repeated elements,
i.e., $(s_1,s_2,\dots,s_k)$ is a permutation of $[k]$.
Further, $\cA$ must stop as soon as it receives a nonzero answer or it would be dominated.
This implies that
 the number of queries issued by $\cA$ when run on $e_i$
 is $t(i) = \min\{ t\ge 1\,:\, s_t = i \}$. (Since $\cA$ is sound and $\epssub \le 1$, $\cA$ cannot stop before time $t(i)$ when running on $e_i$ and if stopped later, it would be dominated by the algorithm that stops at time $t(i)$). 
Since $(s_i)_i$ is a permutation of $[k]$, so is $(t(i))_{i\in [k]}$.
Then, the expected number of queries issued by $\cA$ is 
$\EE{t(I)} = \frac1k \sum_i t(i) = \frac1k \sum_i i = (k+1)/2$.

%%%%%%%%%%%%%%%%%%%%%%%%%%%%%%%%%%%%%%%%%%%%%%%%%%%%%%
% PROOF OF COROLLARY
%%%%%%%%%%%%%%%%%%%%%%%%%%%%%%%%%%%%%%%%%%%%%%%%%%%%%%
\section{Proof of Corollary~\ref{cor:badmx}}\label{app:cor:badmx}
By the proof of \cref{lem:obv}, it should be clear that
if $\delta e_1,\ldots, \delta e_k \in \cH$, then $c_\delta^{\max}(\cH) \geq (k + 1)/2$.
Let
\begin{align*}
k = \floor{\exp\left(\frac{d-1}{8} \left(\frac{\epsilon}{\delta}\right)^2\right)}\,.
\end{align*}
Then, by \cref{lem:jl}, there exists a feature matrix $\Phi' \in \R^{k \times d}$ such that
for all $a \neq b \in \rows(\Phi')$, $\norm{a} = 1$ and $\ip{a, b} \leq \epsilon / \delta$.
Let $\rows(\Phi') = \{a_1,\dots,a_k\}$.
Then, $\norm{\delta e_i-\Phi' (\delta a_i)} \le \epsilon$
and hence $\delta e_i \in \cH^\epsilon_\Phi$ for all $i \in [k]$.
Thus, 
$c_\delta^{\max}(\cH^\epsilon_\Phi) \geq (k+1)/2$.
The result follows from the definition of $k$.

%%%%%%%%%%%%%%%%%%%%%%%%%%%%%%%%%%%%%%%%%%%%%%%%%%%%%%
% GENERAL ANALYSIS OF ESTIMATION PROBLEM
%%%%%%%%%%%%%%%%%%%%%%%%%%%%%%%%%%%%%%%%%%%%%%%%%%%%%%
\section{Proof of Proposition~\ref{prop:bounds}}\label{app:bounds}

\begin{algorithm}
\textsc{input:}\,\, $\Phi \in \R^{k \times d}$ and $q \in [k]$ and $\epsilon \geq 0$
\begin{enumerate}[leftmargin=0.7cm]
\algitem{1}
Find $C \subset [k]$ with $|C| = q$ minimising
\begin{align*}
\max\left\{\norm{\Phi u}_\infty : u \in \R^d, \norm{\Phi_C u} \leq 1\right\}
\end{align*}
\algitem{2} Probe $\mu$ on $C$
\algitem{3} Find $\hat \theta \in \R^d$ such that $\norm{\Phi_C \hat \theta - \mu_C}_\infty \leq \epsilon$
\algitem{4} Return $\hat \mu = \Phi \hat \theta$
\end{enumerate}
\caption{Optimal estimation algorithm}
\label{alg:opt}
\end{algorithm}

\paragraph{Upper bound}
The upper bound is realised by Algorithm~\ref{alg:opt}.
Since $\mu \in \cH^\epsilon_\Phi$, there exists a $\theta \in \R^d$ and $\Delta \in B_\infty(\epsilon)$ such that $\mu = \Phi \theta + \Delta$.
By the definition of the algorithm, $\norm{\hat \mu_C - \mu_C}_\infty \leq \epsilon$, which implies that
\begin{align*}
\norm{\Phi_C(\theta - \hat \theta)}_\infty
&= \norm{\mu_C - \Delta_C - \hat \mu_C}_\infty  \\
&\leq \epsilon + \norm{\mu_C - \hat \mu_C}_\infty \\
&\leq 2\epsilon\,.
\end{align*}
Next, using the definition of $\lambda_q$,
\begin{align*}
\norm{\mu - \hat \mu}_\infty 
&= \norm{\Phi(\theta - \hat \theta) + \Delta}_\infty \\
&\leq \norm{\Phi(\theta - \hat \theta)}_\infty + \epsilon \\
&\leq \norm{\Phi_C(\theta - \hat \theta)}_\infty \max_{v \in \R^d \setminus \{\zeros\}} \frac{\norm{\Phi v}_\infty}{\norm{\Phi_C v}_\infty} + \epsilon \\
&\leq 2\epsilon \lambda_q(\Phi) + \epsilon \\
&< \delta_2\,.
\end{align*}

\paragraph{Lower bound}
Suppose an algorithm is sound with respect to $(\cH^\epsilon_\Phi, \delta_1)$. It suffices to show that there exists a $\mu \in \cH^\epsilon_\Phi$ such that
whenever the algorithm halts with non-zero probability, it has made more than $q$ queries.
Let $\mu = \zeros$ and suppose the algorithm halts having queried $\mu$ on $C \subset [k]$ with non-zero probability and $|C| \leq q$.
Let $\cR = \{\mu \in \cH^\epsilon_\Phi : \mu_C = \zeros\}$ be the set of plausible rewards consistent with the observation. 
By the assumption that the algorithm is sound with respect to $(\cH^\epsilon_\Phi, \delta_1)$, it must be that
\begin{align*}
2 \max_{\nu \in \cR} \norm{\nu}_\infty \leq \max_{\nu, \xi \in \cR} \norm{\nu - \xi}_\infty \leq 2 \delta_1\,,
\end{align*}
where the first inequality is true since for all $\nu \in \cR$ we have $-\nu \in \cR$.
Then,
\begin{align*}
&\max_{\nu \in \cR} \norm{\nu}_\infty \\
&= \max\left\{\norm{\Phi \theta + \Delta}_\infty \!:\! \norm{\Phi_C \theta}_\infty \leq \epsilon, \theta \in \R^d, \Delta \in B_\infty(\epsilon)\right\} \\
&= \epsilon + \max\left\{\norm{\Phi \theta}_\infty \!:\! \norm{\Phi_C \theta}_\infty \leq \epsilon, \theta \in \R^d\right\} \\
&\geq \epsilon + \epsilon \lambda_q(\Phi) \\
&\geq \delta_1\,,
\end{align*}
which contradicts soundness and hence $|C| > q$.

%%%%%%%%%%%%%%%%%%%%%%%%%%%%%%%%%%%%%%%%%%%%%%%%%%%%%%
% STANDARD CALCULATIONS FOR BANDITS
%%%%%%%%%%%%%%%%%%%%%%%%%%%%%%%%%%%%%%%%%%%%%%%%%%%%%%
\newcommand{\bbP}{\mathbb P}
\section{Details for proof of \cref{prop:linear}}\label{app:bandit}
The proof is completed in two steps. First we summarise what has already been established about the within-episode behaviour of the algorithm. Then, in the second step,
the regret is summed over the episodes.

\paragraph{In-episode behaviour}
Let $\sF$ be the $\sigma$-algebra generated by the history up to the start of a given episode and $\hat \mu$ be the least-squares estimate of $\mu$ computed
by the algorithm in that episode. Similarly, let $m$, $u$ and $\cA$ be the quantities defined in the given episode of the algorithm.
Let $a^* = \argmax_{a \in \cA} \mu_a$ and $\hat a = \argmax_{a \in \cA} \hat \mu_a$.
Now, for $b \in \cA$, let $\cE_b$ be the event 
\begin{align*}
\cE_b = \left\{\left|\ip{b, \hat \theta - \theta}\right| \leq 2\epsilon \sqrt{d} + \sqrt{\frac{4d}{m} \log\left(\frac{1}{\alpha}\right)}\right\}\,.
\end{align*}
We have shown that $\bbP(\cE_b \mid \sF) \geq 1 - \alpha$,
which by a union bound implies that
\begin{align*}
\bbP(\cup_{b \in \cA} \cE_b \mid \sF) \geq 1 - k \alpha\,.
\end{align*}
By the definition of the algorithm, an action $a \in \cA$ is not eliminated at the end of the episode if
\begin{align*}
 \ip{\hat \theta, \hat a - a} \leq 2 \sqrt{\frac{4d}{m} \log\left(\frac{1}{\delta}\right)}\,,
\end{align*}
which implies that
\begin{align*}
&2 \sqrt{\frac{4d}{m} \log\left(\frac{1}{\delta}\right)} 
\geq \ip{\hat \theta, \hat a - a} \\
&\geq \ip{\theta, a^* - a} - 2 \sqrt{\frac{4d}{m} \log\left(\frac{1}{\alpha}\right)} - 4\epsilon \sqrt{d}\,.
\end{align*}
Hence, if $a \in \cA$ is not eliminated, then
\begin{align*}
\mu_a 
&\geq \ip{a, \theta} - \epsilon \\
&\geq \ip{a^*, \theta} - \epsilon - 4\sqrt{\frac{4d}{m} \log\left(\frac{1}{\alpha}\right)} - 4\epsilon \sqrt{d} \\
&\geq \mu_{a^*} - \epsilon - 4\sqrt{\frac{4d}{m} \log\left(\frac{1}{\alpha}\right)} - 4\epsilon \sqrt{d}\,.
\end{align*}
Because the condition for eliminating arms does not depend on $\epsilon$, it is not possible to prove that $a^*$ is not eliminated with high probability.
What we can show is that at least one near-optimal action is retained. Suppose that $a^*$ is eliminated, then, using the definition of the algorithm,
\begin{align*}
&2\sqrt{\frac{4d}{m} \log\left(\frac{1}{\alpha}\right)} 
<\ip{\hat \theta, \hat a - a^*} \\
&\leq \ip{\theta, \hat a - a^*} + 2\sqrt{\frac{4d}{m} \log\left(\frac{1}{\alpha}\right)} + 4 \epsilon \sqrt{d}\,.
\end{align*}
Rearranging shows that
\begin{align}
\mu_{\hat a} 
&\geq \ip{\theta, \hat a} - \epsilon \nonumber \\
&> \ip{\theta, a^*} - \epsilon - 4\epsilon \sqrt{d} \nonumber \\
&\geq \mu_{a^*} - 2\epsilon(1 + 2\sqrt{d})\,. \label{eq:keep-good}
\end{align}
Of course, $\hat a$ is not eliminated, which means that either $a^*$ is retained, or an action with nearly the same reward is.
What we have shown is that arms are eliminated if they are much worse than $a^*$ and that some arm with mean close to $a^*$ is retained.
We now combine these results.

\paragraph{Combining the episodes}
Let $L$ be the number of episodes and $\delta_\ell$ be the suboptimality of the best arm in the active set at the start of episode $\ell$.
By the previous part, with probability at least $1 - k \alpha L$, the good events occur in all episodes. Suppose for a moment that this happens. Then, by \cref{eq:keep-good},
\begin{align*}
\delta_\ell \leq 2 \epsilon (\ell - 1) (1 + 2\sqrt{d})\,.
\end{align*}
Then, letting $m_\ell = 2^{\ell-1} \ceil{4d \log \log(d) + 16}$,
the regret is bounded by
\begin{align*}
&R_n = \sum_{\ell=2}^L \sum_{a \in \cA_\ell} u_\ell(a) (\mu^* - \mu_a) \\
&\leq m_1 + \sum_{\ell=2}^L m_\ell \left[\delta_{\ell} + 2\sqrt{\frac{4d}{m_{\ell-1}} \log\left(\frac{1}{\alpha}\right)} + 4\epsilon \sqrt{d}\right] \\
&\leq C\left[\sqrt{d m_L \log\left(\frac{1}{\alpha}\right)} + \epsilon m_L \log(m_L) \sqrt{d}\right] \\
&\leq C\left[\sqrt{d n \log\left(\frac{1}{\alpha}\right)} + \epsilon n \sqrt{d} \log(n)\right]\,.
\end{align*}
The result follows because the regret due to failing confidence intervals is
at most $\alpha k n L \leq L \leq \log_2(n)$, which is negligible relative to the above term.

\begin{remark}\label{rem:known}
When $\epsilon$ is known, the elimination condition can be changed to
\begin{align*}
\cA \leftarrow \left\{a \in \cA : \max_{b \in \cA} \frac{\ip{\hat \theta, b - a}}{4} \leq \sqrt{\frac{d}{m} \log\left(\frac{1}{\delta}\right)} + \epsilon \sqrt{d}\right\}\,.
\end{align*}
Repeating the analysis now shows that the optimal action is never eliminated with high probability, which eliminates the logarithmic dependence in the second
term of \cref{prop:linear}.
\end{remark}

%%%%%%%%%%%%%%%%%%%%%%%%%%%%%%%%%%%%%%%%%%%%%%%%%%%%%%
% OPTIMISM AND LINEAR BANDITS
%%%%%%%%%%%%%%%%%%%%%%%%%%%%%%%%%%%%%%%%%%%%%%%%%%%%%%
\section{Linear contextual bandits}

In the contextual version of the misspecified linear bandit problem, the feature matrix changes from round to round.
Let $(k_t)_{t=1}^n$ be a sequence of natural numbers.
At the start of round $t$ the learner observes a matrix $\Phi_t \in \R^{k_t \times d}$, chooses an action $X_t \in \rows(\Phi_t)$ and receives
a reward $Y_t = \ip{X_t, \theta} + \eta_t + \Delta(X_t)$ where $\Delta : \R^d \to \R$ satisfies $\norm{\Delta}_\infty \leq \epsilon$.
Elimination algorithms are not suitable for such problems.
Here we show that if $\epsilon$ is known, then a simple modification of LinUCB \citep{AST11} can be effective.
You might wonder whether or not this algorithm works well without modification. The answer is sadly negative.

Let $G_t = I + \sum_{s=1}^t X_s X_s^\top$ and define the regurlarised least-squares estimator based on data from the first $t$ rounds by
\begin{align*}
\hat \theta_t = G_t^{-1} \sum_{s=1}^t X_s Y_s\,. 
\end{align*}
Assume for all $a \in \cup_{t=1}^n \rows(\Phi_t)$ that $\norm{a}_2 \leq 1$ and $|\ip{a, \theta}| \leq 1$.
The standard version of LinUCB chooses
\begin{align}
&X_{t+1} = \argmax_{a \in \rows(\Phi_{t+1})} \ip{a, \hat \theta_t} + \norm{a}_{G_t^{-1}} \beta_t\,, \\
&\quad \text{ with } \beta_t = 1 + \sqrt{2 \log\left(n\right) + d \log\left(1 + \frac{n}{d}\right)}\,.
\label{eq:oful}
\end{align}
The modification chooses
\begin{align}
X_{t+1} = \argmax_{a \in \rows(\Phi_{t+1})} \ip{a, \hat \theta_t} + \norm{a}_{G_t^{-1}} \beta_t + \epsilon \sum_{s=1}^t |a^\top G_t^{-1} X_s|\,. 
\label{eq:oful-adapt}
\end{align}
This modification is reminiscent of the algorithm by \citet{JYW19}, who use a bonus of $\Omega(t^{1/2})$ when using upper confidence bounds for least-squares estimators for linear
dynamics in reinforcement learning.

\begin{theorem}
The regret of the algorithm defined by \cref{eq:oful-adapt} satisfies
\begin{align*}
R_n = O\left(d \sqrt{n} \log(n) + n \epsilon \sqrt{d \log(n)}\right)\,.
\end{align*}
\end{theorem}

\begin{proof}[Proof sketch]
The main point is that the additional bonus term ensures optimism.
Then, the standard regret calculation shows that 
\begin{align*}
R_n
&= O\left(d \sqrt{n} \log(n) + \epsilon \E\left[\sum_{t=1}^n \sum_{s=1}^{t-1} |X_t^\top G_{t-1}^{-1} X_s|\right]\right)\,.
\end{align*}
The latter term is bounded by
\begin{align*}
\sum_{t=1}^n \sum_{s=1}^{t-1} |X_t^\top G_{t-1}^{-1} X_s|
&\leq n \sqrt{\sum_{t=1}^n \sum_{s=1}^{t-1} (X_t^\top G_{t-1}^{-1} X_s)^2} \\
&\leq n \sqrt{\sum_{t=1}^n \norm{X_t}^2_{G_{t-1}^{-1}}} \\
&= O\left(n \sqrt{d \log(n)}\right)\,.
\end{align*}
Hence the regret of this algorithm satisfies
\begin{align*}
R_n &= O\left(d \sqrt{n} \log(n) + \epsilon n \sqrt{d \log(n)}\right)\,. \qedhere
\end{align*}
\end{proof}

\begin{remark}
As far as we know, there is no algorithm obtaining a similar bound when $\epsilon$ is unknown.
\end{remark}

\paragraph{Failure of unmodified algorithm}
That the algorithm defined by \cref{eq:oful} is not good for contextual bandits follows from the following example. 
Let $\eta_t = 0$ for all rounds $t$ and 
\begin{align*}
\Phi_{\textrm{odd}} =
\begin{pmatrix}
\epsilon & 0
\end{pmatrix}
\quad
\Phi_{\textrm{even}}
=
\begin{pmatrix}
0 & \epsilon
\end{pmatrix}
\quad
\Phi_{\textrm{large}}
=
\begin{pmatrix}
2 & 1 \\
0 & 0 
\end{pmatrix}\,.
\end{align*}
Now suppose for odd rounds $t \leq n/2$ the feature matrix is $\Phi_t = \Phi_{\textrm{odd}}$ in odd rounds and $\Phi_t = \Phi_{\textrm{even}}$ in even rounds.
For rounds $t > n/2$ the feature matrix is $\Phi_{\textrm{large}}$.
Then let $\theta = (1/2, -1/2)$ and $\Delta((\epsilon, 0)) = -\epsilon$ and $\Delta((0,\epsilon)) = \epsilon$.
Hence for $t = n/2$,
\begin{align*}
G_t &= 
\begin{pmatrix}
1 + n\epsilon^2 / 4 & 0 \\
0 & 1 + n\epsilon^2 / 4
\end{pmatrix} \\
\hat \theta_t &= \left(-\frac{n\epsilon^2/8}{1+n\epsilon^2/4}, \frac{n\epsilon^2/8}{1 + n\epsilon^2/4}\right)\,.
\end{align*}
Therefore $\norm{(2,1)}_{G_t^{-1}}^2 \leq 20 / (n \epsilon^2)$ and
\begin{align*}
\ip{\hat \theta_t, (2, 1)} = -\frac{n \epsilon^2 / 8}{1 + n\epsilon^2 / 4} \leq \frac{4}{n \epsilon^2} - 1\,.
\end{align*}
Hence
\begin{align*}
\ip{\hat \theta_t, (2, 1)} + \norm{(2,1)}_{G_t^{-1}} \beta_t = -1 + O\left(\sqrt{\frac{d \log(n)}{n \epsilon^2}}\right)
\end{align*}
and this for every suitably large $n$ the algorithm will choose $(0,0)$ for all rounds $t \geq n/2$ and suffer regret at least $n/2$.
Thus, if $R_n(\epsilon)$ is the regret on the above problem,
\begin{align*}
\sup_{n,\epsilon>0} \frac{R_n(\epsilon)}{\epsilon n\sqrt{\log(n)}}=\infty\,, 
\end{align*}
while for the modified algorithm,
\begin{align*}
\sup_{n,\epsilon>0} \frac{R_n(\epsilon)}{\epsilon n\sqrt{\log(n)}}<+\infty \,.
\end{align*}

%%%%%%%%%%%%%%%%%%%%%%%%%%%%%%%%%%%%%%%%%%%%%%%%%%%%%%
% LOWER BOUNDS FOR LINEAR BANDITS
%%%%%%%%%%%%%%%%%%%%%%%%%%%%%%%%%%%%%%%%%%%%%%%%%%%%%%
\section{Lower bounds for linear bandits}\label{sec:bandit-lower}

The upper bound in Section~\ref{sec:bandits} cannot be improved in the most interesting regimes, as the following theorem shows: 

\begin{theorem}
There exists a feature matrix $\Phi \in \R^{k \times d}$ such that for any algorithm there is a mean reward vector $\mu \in \cH^\epsilon_\Phi$
for which
\begin{align*}
R_n \geq \epsilon \min(n, (k-1)/2) \sqrt{\frac{d-1}{8 \log(k)}}\,.
\end{align*}
\end{theorem}

\begin{proof}
By the negative result, we may choose $\Phi \in \R^{k \times d}$ such that 
\begin{align*}
\ip{a, a} &= 1 \text{ for all } a \in \rows(\Phi) \\
\ip{a, b} &\leq \sqrt{\frac{8 \log(k)}{d-1}} \text{ for all } a, b \in \rows(\Phi) \text{ with } a \neq b\,.
\end{align*}
Next, let $a^* \in \rows(\Phi)$ and
\begin{align*}
\theta = \delta a^* \qquad \text{with } \delta =  \sqrt{\frac{d-1}{8 \log(k)}} \,,
\end{align*}
which is chosen so that $\mu \in \cH^\epsilon_\Phi$, where
\begin{align*}
\mu_a = 
\begin{cases}
\delta & \text{if } a = a^* \\
0 & \text{otherwise}\,.
\end{cases}
\end{align*}
Let $\tau = \max\{t \leq n : A_s \neq a^* \forall s \leq t\}$.
Then $\E[R_n] \geq \delta \E[\tau]$.
Since the law of the rewards is independent of $a^*$ for $t \leq \tau$, it follows from the randomisation hammer that
$\E[\tau] \geq \min(n, (k-1)/2)$ and the result follows.
\end{proof}

%%%%%%%%%%%%%%%%%%%%%%%%%%%%%%%%%%%%%%%%%%%%%%%%%%%%%%
% COMPUTATION
%%%%%%%%%%%%%%%%%%%%%%%%%%%%%%%%%%%%%%%%%%%%%%%%%%%%%%

\section{Computation complexity}

We briefly discuss the computation complexity of our algorithms here.
Both the bandit and RL algorithms rely on computing a near-optimal design, which is addressed first.

\paragraph{Computing a near-optimal design}
The standard method for computing a near-optimal design is Frank--Wolfe, which in this setting is often attributed to \citet{Fed72}. 
With this algorithm and an appropriate initialisation constant factor approximation of the optimal design can be computed in
$O(k d^2 \log \log(d))$ computations. For more details we recommend chapter 3 of the book by \citet{Tod16}, which also describes a number
of improvements, heuristics and practical guidance.

\paragraph{Bandit algorithm computations}
\cref{alg:elim} has at most $O(\log(n))$ episodes. In each episode it needs to (a) compute a near-optimal design and (b) collect data and find the least-squares estimator and (c) perform
action elimination.
The computation is dominated by finding the near optimal design and computing the covariance matrix $G$, which leads to a total computation of $O(k d^2 \log \log(d) \log(n) + n d^2)$.

\paragraph{RL computations}
The algorithm described in \cref{sec:rl} operates in episodes over $k$ episodes.
In each episode it computes an approximate design and performs $m$ roll-outs of length $n$ from each action in the core set.
Assuming sampling from the generative model is $O(1)$, the total computation, ignoring logarithmic factors, is
\begin{align*}
\tilde O\left(\frac{dA}{\epsilon^2(1 - \gamma)^4} + \frac{SA d^2}{1 - \gamma}\right)\,.
\end{align*}
Dishearteningly, the size of the state space appears in the computation of the optimal design. 
Hence, while the sample complexity of our algorithm is independent of the state space, the computation complexity is not.

\fi

\end{document}